\documentclass[a4paper, 10pt]{article}
\usepackage{a4wide}

\usepackage{amsmath}
\usepackage{amsthm}
\usepackage{xcolor}
\usepackage{thm-restate}

\usepackage{AM-style}

\usepackage{tikz}
\usetikzlibrary{arrows.meta}

\theoremstyle{plain}

\newtheorem{proposition}{Proposition}

\theoremstyle{definition}
\newtheorem{definition}{Definition}

\theoremstyle{remark}
\newtheorem{example}{Example}
\newtheorem{remark}{Remark}
\newtheorem{notation}{Notation}

\usepackage{authblk}

\title{Necessary and Sufficient Explanations in Abstract Argumentation}
\author[1]{AnneMarie Borg}
\author[1,2]{Floris Bex}
\affil[1]{Department of Information and Computing Sciences, Utrecht University}
\affil[2]{Department of Law, Technology, Markets and Society, Tilburg University}
\affil[ ]{\texttt {\{a.borg, f.j.bex\}@uu.nl}}

\begin{document}

\maketitle
\begin{abstract} 
\noindent 
   In this paper, we discuss \emph{necessary} and \emph{sufficient} explanations for formal argumentation --  the question whether and why a certain argument can be accepted (or not) under various extension-based semantics. Given a framework with which explanations for argumentation-based conclusions can be derived, we study necessity and sufficiency: what (sets of) arguments are necessary or sufficient for the (non-)acceptance of an argument? 
\end{abstract}

\section{Introduction}

%As we leave more decisions to AI systems it is important that we understand how certain conclusions are reached, especially when users have to be able to interpret, understand and trust these systems. Therefore, it is essential that AI systems can explain their decisions. 
In recent years, \emph{explainable AI} (XAI) has received much attention, mostly directed at new techniques for explaining decisions of (subsymbolic) machine learning algorithms \cite{samek2017explainable}. However, explanations traditionally also play an important role in (symbolic) knowledge-based systems \cite{lacave2004review}. Computational argumentation is one research area in symbolic AI that is frequently mentioned in relation to XAI. For example, arguments can be used to provide reasons for or against decisions \cite{lacave2004review,ABGHPRSTV17,Mil19}. The focus can also be on the argumentation itself, where it is explained whether and why a certain argument or claim can be accepted under certain semantics for computational argumentation \cite{FaTo15AAAI,FaTo15TAFA,GCRS13,saribaturWW20}. It is the latter type of explanations that is the subject of this paper.

Abstract argumentation frameworks, as introduced in~\cite{Dung95}, consist of sets of arguments (abstract entities) and an attack relation between them. To determine the conclusions of a framework their corresponding extensions -- sets of arguments that can collectively be considered as acceptable -- are calculated under different semantics~\cite{Dung95}. Many of the well-known and most studied semantics are based on the notion of \emph{defense}: an argument is defended by a set of arguments if that set attacks all its attackers. 

%In this paper we apply a basic framework to explain argumentation-based conclusions (i.e., why an argument is (not) part of an or all extensions) derived from abstract argumentation frameworks that can be evaluated by any extension-based semantics. 
In this paper we investigate explanations for ar\-gu\-men\-ta\-tion-based conclusions (i.e., why an argument is (not) part of an or all extensions), by applying a basic framework for explanations on top of abstract argumentation frameworks which can be evaluated by any extension-based semantics. 
%In order to explain the (non-)acceptance of an argument (i.e., why an argument is (not) part of an or all extensions), we aim to provide 
The explanations are defined in terms of sets of relevant arguments that are part of extensions and explain the (non-)ac\-cep\-tance in terms of defense. We consider an argument \emph{relevant} for another argument if the arguments are connected by means of the attack relation. By requiring relevance of an explanation, arguments that do not (in)directly attack or defend the considered argument (and therefore do not influence the acceptability of that argument) will not be part of the explanation. %That we only consider the notion of defense somewhat limits the possible explanations.
Since defense is a central notion in many Dung-style semantics, the explanations thus defined can be applied to all the common semantics (e.g., complete, grounded, preferred). 

One of the important characteristics of explanations provided by humans is that they select \emph{the} explanation from a possible infinite set of explanations, using criteria such as simplicity, necessity and sufficiency~\cite{Mil19}. In this paper we look at how to select minimal\footnote{Interpreting~\cite{Mil19} simplicity as minimality}, necessary and sufficient explanations for the (non-)ac\-cep\-tance of an argument. 

After introducing some preliminary notions concerning (defense) in abstract argumentation frameworks, we provide a \emph{basic framework} with which explanations for both acceptance and non-acceptance of an argument can be provided, given any existing extension-based semantics that is based on the notion of defense. We show that these explanations are well-behaved with respect to different Dung-style semantics (e.g. explanations under grounded semantics are never supersets of explanations under preferred semantics), and discuss notions of \emph{minimality} introduced in~\cite{FaTo15AAAI} applied to explanations in our framework. 

We continue by discussing the notions of \emph{necessity} and \emph{sufficiency}, introducing sufficient and necessary explanations for (non-)acceptance. For these explanations we show when they exists and how the explanations provided by the basic framework are related to necessary and sufficient explanations. Furthermore, we show how our notions of necessity and sufficiency relate to the notions of minimality from~\cite{FaTo15AAAI}. We conclude with discussing related and future work.

\section{Preliminaries}
\label{sec:Preliminaries}

An \emph{abstract argumentation framework\/} (AF)~\cite{Dung95} is a pair $\calAF = \tuple{\Args, \attack}$, where $\Args$ is a set of \emph{arguments\/} and $\attack\subseteq\Args\times\Args$ is an \emph{attack relation\/} on these arguments. Such a framework can be viewed as a directed graph, in which the nodes represent arguments and the arrows represent attacks between arguments, see e.g., Figure~\ref{fig:AAAIRunning}. %\fb{Als het toch abstract blijft misschien wat attacks/arguments weghalen voor de leesbaarheid}

\begin{figure}[ht]
    \centering
    \begin{tikzpicture}[scale = 0.95]
        \node [circle,draw] at (0,0) (A) {$A$};
        \node [circle,draw] at (2,0) (B) {$B$};
        \node [circle,draw] at (4,1) (C) {$C$};
        \node [circle,draw] at (6,1) (D) {$D$};
        \node [circle,draw] at (4,-1) (E) {$E$};
        \node [circle,draw] at (6,-1) (F) {$F$};
        \node [circle,draw] at (8,-1) (G) {$G$};
        
        \draw[->] (B) to (A);
        \draw[->] (C) to (B);
        \draw[->] (C) to [bend right] (D);
        \draw[->] (D) to [bend right] (C);
        \draw[->] (E) to (B);
        \draw[->] (F) to (E);
        \draw[->] (F) to [bend right] (G);
        \draw[->] (G) to [bend right] (F);
    \end{tikzpicture}
    \caption{Graphical representation of the AF $\calAF_1$.}
    \label{fig:AAAIRunning}
\end{figure}
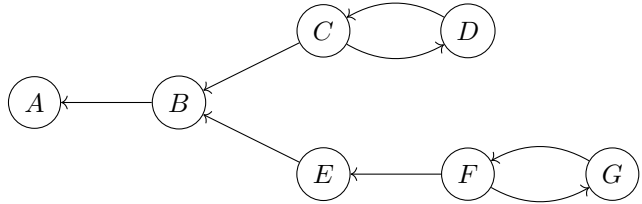

\begin{example}
  \label{ex:abstractAF}
  Figure~\ref{fig:AAAIRunning} represents $\calAF_1 = \tuple{\Args_1,\attack_2}$ where $\Args_1 = \{A,\allowbreak B,\allowbreak C,\allowbreak D,\allowbreak E,\allowbreak F,\allowbreak G\}$ and $\attack_1 = \{(B,A),\allowbreak (C,B),\allowbreak (C,D),\allowbreak (D,C),\allowbreak (E,B), \allowbreak (F,E),\allowbreak (F,G),\allowbreak (G,F)\}$. 
\end{example}

%  \begin{figure}[ht]
%    \centering
%    %\small
%    \begin{tikzpicture}%[scale=0.85]
%      \node [circle, draw] at (2,0) (a1) {$A_1$};
%      \node [circle, draw] at (4,0) (a2) {$A_2$};
%      \node [circle, draw] at (6,0) (a3) {$A_3$};
%      \node [circle, draw] at (8,0) (a4) {$A_4$};
%      
%      \node [circle, draw] at (3,-2) (b1) {$B_1$};
%      \node [circle, draw] at (5,-2) (b2) {$B_2$};
%      \node [circle, draw] at (7,-2) (b4) {$B_3$};
%
%      \draw [->] (b1) to [bend right] (a2);
%      \draw [->] (b1) to (a3);
%      \draw [->] (b1) to (a4);
%      \draw [->] (b4) to [bend left =10](a3);
%      \draw [->] (a3) to [bend left =10](b4);
%      \draw [->] (b4) to (a4);
%      \draw [->] (a2) to [bend right] (b1);
%      \draw [->] (a2) to (b2);
%      \draw [->] (a2) to (b4);
%    \end{tikzpicture}
%    \caption{Graphical representation of the AF $\calAF_1$.}
%    \label{fig:AAAIRunning}
%  \end{figure} 
%
%\begin{example}
%  \label{ex:abstractAF}
%  Figure~\ref{fig:AAAIRunning} represents $\calAF_1 = \tuple{\Args_1,\attack_1}$ where $\Args_1 = \{A_1,\allowbreak A_2,\allowbreak A_3,\allowbreak A_4,\allowbreak B_1,\allowbreak B_2,\allowbreak B_3\}$ and $\attack_1 = \{(A_2,B_1),\allowbreak(A_2,B_2),\allowbreak(A_2,B_3), \allowbreak(A_3,\allowbreak B_3),\allowbreak(B_1,\allowbreak A_2),\allowbreak(B_1,A_3)\allowbreak (B_1,A_4),\allowbreak (B_3,A_3), \allowbreak (B_3,A_4)\}$. 
%\end{example}

Given an AF, Dung-style semantics~\cite{Dung95} can be applied to it, to determine what combinations of arguments (called \emph{extensions}) can collectively be accepted. 

\begin{definition}
  \label{def:extension}
  Let $\calAF=\tuple{\Args,\attack}$ be an AF, $\sfS\subseteq\Args$ a set of arguments and let $A\in\Args$. Then $\sfS$ \emph{attacks} $A$ if there is an $A'\in\sfS$ such that $(A',A)\in\attack$; $\sfS$ \emph{defends} $A$ if $\sfS$ attacks every attacker of $A$; $\sfS$ is \emph{conflict-free} if there are no $A_1,A_2\in\sfS$ such that $(A_1,A_2)\in\attack$; and $\sfS$ is \emph{admissible} ($\adm$) if it is conflict-free and it defends all of its elements. 

An admissible set that contains all the arguments that it defends is a \emph{complete extension} ($\cmp$) of $\calAF$. The \emph{grounded extension} ($\grd$) is the minimal (with respect to $\subseteq$) complete extension. A \emph{preferred extension} ($\prf$) is a maximal (with respect to $\subseteq$) complete extension. A \emph{stable extension} ($\stb$) is a complete extension that attacks every argument not in it. $\exts_\sem(\calAF)$ denotes the set of all the extensions of $\calAF$ under the semantics $\sem\in\{\adm,\cmp,\allowbreak\grd,\allowbreak\prf,\allowbreak\stb\}$.%\footnote{Due to space restrictions we only mention four semantics here. As noted in Remark~\ref{rem:Semantics} the explanations that are introduced can also be applied to other Dung-style semantics such as those from~\cite{BCG18}.} 
\end{definition}

Where $\calAF = \tuple{\Args,\attack}$ is an AF, $\sem$ a semantics and $\exts_\sem(\calAF)\neq\emptyset$, it is said that $A\in\Args$ is \emph{skeptically} [resp.\ \emph{credulously}] \emph{accepted} if $A\in\bigcap\exts_\sem(\calAF)$ [resp.\ $A\in\bigcup\exts_\sem(\calAF)$]. These acceptability strategies are denoted by $\cap$ [resp.\ $\cup$]. $A$ is said to be \emph{credulously} [resp.\ \emph{skeptically}] \emph{non-accepted} in $\calAF$ if for all [resp.\ some] $\ext\in\exts_\sem(\calAF)$,  $A\notin\ext$. We will say that an argument is accepted [resp.\ non-accepted] if the strategy is arbitrary or clear from the context.

The notions of attack and defense can also be defined between arguments and can be generalized to indirect versions:

\begin{definition}
  \label{def:DefenseExtension}
  Let $\calAF = \tuple{\Args,\attack}$ be an AF, $A,B\in\Args$ and $\ext\in\exts_\sem(\calAF)$ for some $\sem$. Then  $A$ \emph{defends} $B$ if: there is some $C\in\Args$ such that $(C,B)\in\attack$ and $(A,C)\in\attack$, in this case $A$ \emph{directly defends $B$}; or $A$ defends $C\in\Args$ and $C$ defends $B$, in this case  $A$ \emph{indirectly defends $B$}. It is said that \emph{$A$ defends $B$ in $\ext$} if $A$ defends $B$ and $A\in\ext$.

  Similarly, $A$ \emph{attacks} $B$ if: $(A,B)\in\attack$, in this case $A$ \emph{directly attacks} $B$; or $A$ attacks some $C\in\Args$ and $C$ defends $B$, in this case $A$ \emph{indirectly attacks} $B$.
\end{definition}

We will require that an explanation for an argument $A$ is \emph{relevant}, in order to prevent that explanations contain arguments that do not influence the acceptance of $A$.

\begin{definition} 
    \label{def:relevance}
    Let $\calAF = \tuple{\Args,\attack}$ and $A,B\in\Args$. It is said that $A$ is \emph{relevant} for $B$ if $A$ (in)directly attacks or defends $B$ and it does not attack itself. A set $\sfS\subseteq\Args$ is relevant for $B$ if all of its arguments are relevant for $B$. 
\end{definition}

%In the context of explanations, requiring relevance of an explanation is a step towards a local explanation (recall the discussion in the introduction). For example, consider an AF with arguments $\{A,B,C,D\}$ and attacks $\{(A,B),(B,A),(C,D),(D,C)\}$. Then there are four preferred extensions: $\{A,C\}$, $\{A,D\}$, $\{B,C\}$ and $\{B,D\}$. For the (non-)acceptance of $A$ and $B$ the arguments $C$ and $D$ are not relevant and should therefore not be part of the explanation. Note that the introduced possibilities for $\depth$, based on the notion of defense, result in relevant explanations. 

\begin{example}
  \label{ex:abstractAFextensions}
  In $\calAF_1$ (from Figure~\ref{fig:AAAIRunning}) $C$ and $D$ attack each other and both defend themselves from this attack. The grounded extension is $\emptyset$ and $\exts_\sem(\calAF_1) = \{\{A,\allowbreak C,\allowbreak E,\allowbreak G\},\allowbreak \{A,\allowbreak C,\allowbreak F\},\allowbreak \{A,\allowbreak D,\allowbreak E,\allowbreak G\},\allowbreak \{B,\allowbreak D,\allowbreak F\}\}= \{\ext_1,\allowbreak \ext_2,\allowbreak \ext_3,\allowbreak \ext_4\}$ for $\sem\in\{\prf,\stb\}$. None of the arguments from $\Args_1$ is skeptically accepted, while all of the arguments in $\Args_1$ are credulously accepted for $\sem\in\{\cmp,\prf,\stb\}$. The argument $G$ defends $E$ directly and $A$ indirectly and, similarly, it attacks $F$ directly and $B$ indirectly. %\fb{moeten A3 en C3 hier niet omgedraaid worden?}
  The arguments $D$ and $G$ are relevant for $A$ and $B$ but not relevant for each other.
\end{example}

In order to define relevant explanations, we focus on the notion of defense. Note that many of the well-known Dung-style semantics result in defended sets of arguments. We define two notions that will be used in the basic definitions of explanations. The first, used for acceptance explanations, denotes the set of arguments that defend the argument $A$, while the last, used for non-acceptance explanations, denotes the set of arguments that attack $A$ and for which there is no defense in the given extension. 

\begin{definition}
  \label{def:GenDefAtt}
  Let $\calAF = \tuple{\Args,\attack}$ be an AF, $A\in\Args$ and $\ext\in\exts_\sem(\calAF)$ an extension for some semantics $\sem$. 
  \begin{itemize}
    \item $\DefBy(A) = \{B\in\Args\mid B \text{ (in)directly defends } A\}$;
    \item $\DefBy(A,\ext) = \DefBy(A)\cap\ext$ denotes the set of arguments that (in)directly defend~$A$ in $\ext$ ;
    \item $\NotDef(A,\ext) = \{B\in \Args \mid B \text{ (in)directly }\allowbreak\text{attacks }A \allowbreak\text{ and } \allowbreak \ext \allowbreak \text{ does not attack }B\}$, denotes the set of all (in)direct attackers of $A$ for which no defense exists from $\ext$.
  \end{itemize}
\end{definition}

Note that by definition explanations in terms of $\DefBy$ or $\NotDef$ result in relevant explanations.

\begin{example}[Example~\ref{ex:abstractAFextensions} continued]
  \label{ex:DefBy}
  Consider again the AF $\calAF_1$. Then: $\DefBy(A,\ext_1) = \{C,E,G\}$, $\DefBy(A,\ext_2) = \{C\}$, while $\NotDef(A,\ext_4) = \{B,D,F\}$ and $\NotDef(B,\ext_1) = \{C,E,G\}$.
\end{example}

We introduce the following notation to keep the notation of the explanations general and short. 

\begin{notation}
  \label{not:SetsExts}
  Let $\calAF = \tuple{\Args,\attack}$ be an AF, $A\in\Args$ and $\sfS\subseteq\Args$. Then, for some $\sem$:
  \begin{itemize}
    \item $\allext_A^\sem = \{\ext\in\exts_\sem(\calAF) \mid A\in\ext\}$ denotes the set of $\sem$-extensions of $\calAF$ of which $A$ is a member;
    \item $\allext_{\not A}^\sem = \{\ext\in\exts_\sem(\calAF) \mid A\notin\ext\}$ denotes the set of $\sem$-extensions of $\calAF$ of which $A$ is not a member;
  \end{itemize}
\end{notation}

Now the basic framework for explanations in formal argumentation can be defined. 

\section{Basic Explanations}
\label{sec:Basic}

In this section we present the basic framework with which a variety of explanations can be provided. Explanations are defined in terms of a function $\depth$, which determines how ``far away'' we should look when considering attacking and defending arguments as explanations. $\depth$ can be adjusted according to, for example, the application. Examples of adjustments are shown in Sections~\ref{sec:NecSuff:Acc} and~\ref{sec:NecSuff:NonAcc}. 

\subsection{Basic Explanations for Acceptance}
\label{sec:Basic:accept}

We define two types of acceptance explanations: $\cap$-ex\-pla\-na\-tions provide all the reasons why an argument can be accepted by a skeptical reasoner, while $\cup$-explanations provide one reason why an argument can be accepted by a credulous reasoner.
For the purpose of this section let $\depth^\acc = \DefBy$. This means that the presented explanations will be based on the set of arguments that defend the requested argument. 

%An argument explanation for an accepted argument $A$ will consist of a set of arguments that defend it in some extension. What extensions are considered depends on the acceptability strategy. 

\begin{definition}[Argument explanation]
  \label{def:Expl:FullAbstract}
  Let $\calAF = \tuple{\Args,\attack}$ be an AF and let $A\in\Args$ be an argument that is accepted, given some $\sem$ and an acceptance strategy ($\cap$ or $\cup$). Then: 
  \begin{align*}
      \Acc_\sem^\cap(A) &= \bigcup_{\ext\in\exts_\sem(\calAF)}\depth^\acc(A,\ext);\hspace*{2cm}\\
      \Acc_\sem^\cup(A) &\in \{\depth^\acc(A,\ext)\mid \ext\in\allext_A^\sem\}.
  \end{align*}
  Let $\texttt{Acc}^\cup_\sem(A) = \{\depth^\acc(A,\ext)\mid \ext\in\allext^\sem_A\}$.
\end{definition}

$\Acc^\cap_\sem(A)$ provides for each $\sem$-extension $\ext$ the arguments that defend $A$ in $\ext$, and $\Acc^\cup_\sem(A)$ contains the arguments that defend $A$ in one of the $\sem$-extensions.

\begin{remark} 
  \label{rem:DefensenotDefended}
  A non-empty $\DefBy$-set for an attacked argument $A$ does not guarantee that $A$ is defended against all its attackers. In the AF $\calAF = \langle\{A,\allowbreak B,\allowbreak C\},\allowbreak \{(A,\allowbreak B),\allowbreak (B, C),\allowbreak (C,\allowbreak A)\}\rangle$ we have that $\DefBy(A) = \{A,\allowbreak B,\allowbreak C\}$ but $A \notin\bigcup\exts_\cmp(\calAF)$ (since it is not conflict-free with $B$ nor with $C$). In the above definition of explanations this issue does not occur, since there it is required that the argument is part of the extension and thus, by definition of the semantics, it is defended against all its attackers by that extension. 
\end{remark}

%\begin{remark}
%  \label{rem:Semantics}
%  The definition of argument explanations, as well as the non-acceptance explanations in the next section, are semantics independent. Other semantics from e.g.,~\cite{BCG18} can be used as well.  
%\end{remark}

\begin{example}[Example~\ref{ex:DefBy} continued]
  \label{ex:Expl:FullAbstract:Abstract}
  Recall $\calAF_1 = \tuple{\Args_1,\attack_1}$, shown in Figure~\ref{fig:AAAIRunning}. We have that: 
  \begin{itemize}
    \item $\Acc^\cup_\prf(A) \in \{\{C,E,G\}, \{C\}, \{E,G\}\}$; and
    \item $\Acc^\cup_\prf(B) = \{D,F\}$.
  \end{itemize}
\end{example}

The next proposition shows that the acceptance explanations, under $\depth = \DefBy$, are well-behaved when compared to Dung-style semantics. For example it shows that explanation under grounded semantics results in  smaller and more skeptical explanations than under preferred semantics. 

\begin{restatable}{proposition}{propAccSem}
  \label{prop:Acc:Semantics}
  Let $\calAF = \tuple{\Args,\attack}$ be an argumentation framework, let $A\in\Args$ and let $\depth = \DefBy$, then:
  \begin{enumerate}
    \item For all $\sem\in\{\grd,\cmp,\prf,\stb\}$ and for $\star\in\{\cup,\cap\}$: $\Acc_\grd(A)\subseteq \Acc_\sem^\star(A)$.
    \item $\Acc^\cap_\stb(A) \subseteq\Acc^\cap_\prf(A)\subseteq\Acc^\cap_\cmp(A)$.
    \item $\texttt{Acc}_\stb^\cup(A)\subseteq\texttt{Acc}_\prf^\cup(A)\subseteq\texttt{Acc}_\cmp^\cup(A)$.  
    \item For each $\sfS\in\texttt{Acc}^\cup_\cmp(A)$ there is an $\sfS'\in\texttt{Acc}^\cup_\prf(A)$ such that $\sfS\subseteq\sfS'$
  \end{enumerate}
\end{restatable}

\begin{proof}
  Let $\calAF = \tuple{\Args,\attack}$ be an argumentation framework, let $A\in\Args$ be accepted and let $\depth = \DefBy$.
  \begin{enumerate}
    \item Let $B\in\Acc_\grd(A)$, then \(B\) (in)directly defends \(A\) and $B\in\exts_\grd(\calAF)$, therefore $B\in\bigcap\exts_\sem(\calAF)$. It follows that $B\in\DefBy(A,\ext)$ for all $\ext\in\exts_\sem(\calAF)$. Hence $B\in\Acc_\sem^\star(A)$ as well. 
    \item Let $B\in\Acc^\cap_\stb(A)$, then there is some $\ext\in\exts_\stb(\calAF)$ such that $B\in\DefBy(A,\ext)$. Since $\exts_\stb(\calAF)\subseteq\exts_\prf(\calAF)\subseteq\exts_\cmp(\calAF)$ it follows that $\ext\in\exts_\sem(\calAF)$ for $\sem\in\{\prf,\cmp\}$ as well. Therefore $B\in\Acc^\cap_\prf(A)$ and $B\in\Acc^\cap_\cmp(A)$. 
    \item Let $\sfS\in\texttt{Acc}_\stb^\cup(A)$, then $\sfS = \DefBy(A,\ext)$ for some $\ext\in\exts_\stb(\calAF)\subseteq\exts_\prf(\calAF)\subseteq\exts_\cmp(\calAF)$. Therefore $\sfS\in\texttt{Acc}_\prf^\cup(A)$ and $\sfS\in\texttt{Acc}_\cmp^\cup(A)$ as well. 
    \item Let $\sfS\in\texttt{Acc}^\cup_\cmp(A)$, then there is some $\ext\in\exts_\cmp(\calAF)$ such that $\sfS = \DefBy(A,\ext)$. By Definition~\ref{def:extension} there is some $\ext\subseteq\ext'\in\exts_\prf(\calAF)$, let $\sfS' = \DefBy(A,\ext')$. Note that $\sfS'\in\texttt{Acc}^\cup_\prf(A)$ and that $\sfS\subseteq\sfS'$. 
    \qedhere
  \end{enumerate}
\end{proof}

%\begin{proof}
%  Let $\calAF = \tuple{\Args,\attack}$ be an AF, let $A\in\Args$ be accepted and let $\depth = \DefBy$. We show the first two items:\footnote{Due to space restrictions we cannot provide full proofs for all the results in this paper. We refer the interested reader to the supplementary material in which full proofs are collected.}
%  \begin{enumerate}
%    \item Let $B\in\Acc_\grd(A)$, then $B$ (in)directly defends $A$ and $B\in\exts_\grd(\calAF)$, therefore $B\in\bigcap\exts_\sem(\calAF)$. It follows that $B\in\DefBy(A,\ext)$ for all $\ext\in\exts_\sem(\calAF)$. Hence $B\in\Acc_\sem^\star(A)$ as well. 
%    \item Let $B\in\Acc^\cap_\stb(A)$, then there is some $\ext\in\exts_\stb(\calAF)$ such that $B\in\DefBy(A,\ext)$. Since $\exts_\stb(\calAF)\subseteq\exts_\prf(\calAF)\subseteq\exts_\cmp(\calAF)$ it follows that $\ext\in\exts_\sem(\calAF)$ for $\sem\in\{\prf,\cmp\}$ as well. Therefore $B\in\Acc^\cap_\prf(A)$ and $B\in\Acc^\cap_\cmp(A)$. 
%%    \item Let $\sfS\in\Acc_\stb^\cup(A)$, then $\sfS = \DefBy(A,\ext)$ for some $\ext\in\exts_\stb(\calAF)\subseteq\exts_\prf(\calAF)\subseteq\exts_\cmp(\calAF)$. Therefore $\sfS\in\Acc_\prf^\cup(A)$ and $\sfS\in\Acc_\cmp^\cup(A)$as well. 
%%    \item Let $\sfS\in\Acc^\cup_\cmp(A)$, then there is some $\ext\in\exts_\cmp(\calAF)$ such that $\sfS = \DefBy(A,\ext)$. By Definition~\ref{def:extension} there is some $\ext\subseteq\ext'\in\exts_\prf(\calAF)$, let $\sfS' = \DefBy(A,\ext')$. Note that $\sfS'\in\Acc^\cup_\prf(A)$ and that $\sfS\subseteq\sfS'$. 
%    \qedhere
%  \end{enumerate}
%\end{proof}

\subsection{Basic Explanations for Non-Acceptance}
\label{sec:Basic:nonaccpt}

Understanding why something is not accepted might sometimes be just as important as understanding why something is accepted. We therefore discuss basic definitions for explanations of non-accepted arguments. Since we focus on the notion of defense and admissibility-based semantics, an argument is not accepted if it is attacked and there is no defense for this attack by an accepted argument. In this section, let $\depth^\nacc = \NotDef$. %From an argumentation perspective, the explanation for the non-acceptance of an argument is the set of arguments for which no defense exists. 

\begin{definition}[Non-acceptance explanation]
  \label{def:Expl:NonAcc:Abstract}
  Let $\calAF = \tuple{\Args,\attack}$ be an AF and let $A\in\Args$ be an argument that is not accepted, given some $\sem$ and some $\star\in\{\cap,\cup\}$. %Then:
  \begin{align*}
    \NotAcc_\sem^\cap(A) &= \bigcup_{\ext\in\allext_{\not A}^\sem}\depth^\nacc(A,\ext);\hspace*{2.2cm} \\
    \NotAcc_\sem^\cup(A) &= \bigcup_{\ext\in\exts_\sem(\calAF)} \depth^\nacc(A,\ext).
  \end{align*}
\end{definition}

Intuitively, a non-acceptance explanation contains all the arguments in $\Args$ that attack $A$ and for which no defense exists in: some $\sem$-extensions (for $\cap$) of which $A$ is not a member; all $\sem$-extensions (for $\cup$). That for $\cap$ only some extensions have to be considered follows since $A$ is skeptically non-accepted as soon as $\allext_{\not A}^\sem\neq \emptyset$, while $A$ is credulously non-accepted when $\allext_{\not A}^\sem =\exts_\sem(\calAF)$. 

\begin{example}(Example~\ref{ex:DefBy} continued)
  \label{ex:Expl:NonAcc:Abstract}
  Recall the argumentation framework from Example~\ref{ex:abstractAF}, shown in Figure~\ref{fig:AAAIRunning}. Then:
  \begin{itemize}
    \item $\NotAcc_\grd^\cap(A) = \NotAcc_\prf^\cap(A)= \{B,D,F\}$; and
    \item $\NotAcc^\cap_\grd(B) = \NotAcc_\prf^\cap(B) = \{C,E,G\}$.
  \end{itemize}
\end{example}

The next proposition is the non-acceptance counterpart of Proposition~\ref{prop:Acc:Semantics} and shows how non-acceptance explanations are related to each other under different semantics.

\begin{restatable}{proposition}{propNonAccSem}
  \label{prop:NonAcc:Semantics}
  Let $\calAF = \tuple{\Args,\attack}$ be an AF, let $A\in\Args$ be non-accepted, $\depth = \NotDef$ and let $\star\in\{\cap,\cup\}$. Then:
  \begin{enumerate}
    \item $\NotAcc^\cap_\grd(A) \subseteq \NotAcc^\cap_\cmp(A)$. 
    \item $\NotAcc^\star_\stb(A) \subseteq\NotAcc^\star_\prf(A)\subseteq\NotAcc^\star_\cmp(A)$. 
  \end{enumerate}
\end{restatable}

\begin{proof}
  Let $\calAF = \tuple{\Args,\attack}$ be an argumentation framework, let $A\in\Args$ be non-accepted and let $\depth = \NotDef$. 
  \begin{enumerate}
    \item Let $\ext_\grd \in\exts_\grd(\calAF)$ be the grounded extension, note that $\ext_\grd = \bigcap\exts_\cmp(\calAF)$ and that $\ext_\grd\in\exts_\cmp(\calAF)$~\cite{Dung95}. It therefore follows that, if $A\notin\ext_\grd$, then $\ext_\grd\in\allext^\cmp_{\not A}$. By Definition~\ref{def:Expl:NonAcc:Abstract} we have immediately that $\NotAcc^\cap_\grd(A) \subseteq\NotAcc^\cap_\cmp(A)$. 
    \item Suppose that $A$ is not accepted w.r.t.\ $\stb$ and $\cap$, then for some $\ext\in\exts_\stb(\calAF)$, $A\notin\ext$. Note that $\allext^\stb_{\not A}\subseteq\allext^\prf_{\not A}\subseteq\allext^\cmp_{\not A}$. Therefore, for all $\ext\in\allext^\stb_{\not A}$, $\ext\in\allext^\prf_{\not A}$ and $\ext\in\allext^\cmp_{\not A}$. Hence $\NotAcc^\cap_\stb(A)\subseteq\NotAcc^\cap_\prf(A)\subseteq\NotAcc^\cap_\cmp(A)$. 
    
    The case for $\cup$ is similar and left to the reader. Note that, by assumption, $A$ is not accepted w.r.t.\ $\cmp$ and $\cup$, otherwise $\NotAcc^\cup_\cmp(A)$ is not defined. \qedhere
  \end{enumerate}
\end{proof}
%
%\begin{proof}
%  Let $\calAF = \tuple{\Args,\attack}$ be an argumentation framework, let $A\in\Args$ be non-accepted and let $\depth = \NotDef$. We show the second item. 
%%  \begin{enumerate}
%%    \item Let $\ext_\grd \in\exts_\grd(\calAF)$ be the grounded extension, note that $\ext_\grd = \bigcap\exts_\cmp(\calAF)$ and that $\ext_\grd\in\exts_\cmp(\calAF)$~\cite{Dung95}. It therefore follows that, if $A\notin\ext_\grd$, then $\ext_\grd\in\allext^\cmp_{\not A}$. By Definition~\ref{def:Expl:NonAcc:Abstract} we have immediately that $\NotAcc^\cap_\grd(A) \subseteq\NotAcc^\cap_\cmp(A)$. 
%%    \item 
%Suppose that $A$ is not accepted w.r.t.\ $\stb$ and $\cap$, then for some $\ext\in\exts_\stb(\calAF)$, $A\notin\ext$. Note that $\allext^\stb_{\not A}\subseteq\allext^\prf_{\not A}\subseteq\allext^\cmp_{\not A}$. Therefore, for all $\ext\in\allext^\stb_{\not A}$, $\ext\in\allext^\prf_{\not A}$ and $\ext\in\allext^\cmp_{\not A}$. Hence $\NotAcc^\cap_\stb(A)\subseteq\NotAcc^\cap_\prf(A)\subseteq\NotAcc^\cap_\cmp(A)$. 
%    
%    The case for $\cup$ is similar and left to the reader. Note that, by assumption, $A$ is not accepted w.r.t.\ $\cmp$ and $\cup$, otherwise $\NotAcc^\cup_\cmp(A)$ is not defined. %\qedhere
%%  \end{enumerate}
%\end{proof}

Proposition~\ref{prop:Acc:Semantics} together with the above result shows that the choice of the semantics influences the size of the explanation in a similar way as the choice of semantics influences the size and number of extensions themselves. This is useful to know, since it shows that the explanations behave in a predictable way and that semantics can be chosen as usual. In the next section we will look at minimal explanations.

\section{Minimality}
\label{sec:Minimality}

As mentioned in the introduction, humans select \emph{the} explanation from all the possible explanations, using criteria such as simplicity, necessity and sufficiency~\cite{Mil19}. One way to look at simplicity is minimality.
%
%In some situations there might be many possible explanations, however, humans are able to select just one or a few from these and present these as \emph{the} explanations~\cite{Mil19}. They do so by selecting, based on e.g., the context of the situation. One way of selecting explanations is by looking at explanations that are minimal or compact.
%
In~\cite{FaTo15AAAI} two notions of minimality were introduced (as well as two notions of maximality, but we are in this paper only interested in minimality): minimality (i.e., minimality w.r.t.\ $\leq$)\footnote{Where minimality w.r.t.\ $\leq$ is applied to the size of sets: $\sfS_1\leq\sfS_2$ denotes $|\sfS_1|\leq|\sfS_2|$} and compactness (i.e., minimality w.r.t.\ $\subseteq$). In our setting we can formulate different minimal explanations for (non-)acceptance as follows, where $\mima \in \{\min^\leq,\min^\subseteq\}$.
\begin{align*}
  &\Acc^\cap_\sem(A) = \underset{\ext\in\exts_\sem(\calAF)}{\mima}\depth^\acc(A,\ext);\hspace*{2.5cm}\\
  &\Acc^\cup_\sem(A) = \underset{\ext\in\allext_A^\sem}{\mima}\ \depth^\acc(A,\ext);\\
  &\NotAcc^\cap_\sem(A) = \underset{\ext\in\allext^\sem_{\not A}}{\mima}\ \depth^\nacc(A,\ext); \\
  &\NotAcc^\cup_\sem(A) = \underset{\ext\in\exts_\sem(\calAF)}{\mima}\depth^\nacc(A,\ext).
\end{align*}

\begin{example}[Examples~\ref{ex:Expl:FullAbstract:Abstract} and~\ref{ex:Expl:NonAcc:Abstract} continued]
    \label{ex:MonCom}
    Recall that, for $\calAF_1$, $\Acc^\cup_\prf(A) \in \{\{C,E,G\},\{C\}, \allowbreak\{E,G\}\}$. Of these possible explanations  both $\{C\}$ and $\{E,G\}$ are $\subseteq$-minimal, but only $\{C\}$ is $\leq$-minimal. Similarly, we had that $\NotAcc^\cap_\prf(B) = \{C,E,G\}$, there are two $\subseteq$-minimal explanations: $\{C\}$ and $\{E,G\}$ but only $\{C\}$ is also a $\leq$-minimal explanation. 
\end{example}

These notions of minimality are already useful in restricting the size of an explanation. As we have seen in the example above, if \(\mima = \min^\subseteq\) $\Acc^\cup_\prf(A)\in \{\{C\}, \{E,G\}\}$. It is therefore no longer the case that $\Acc^\cup_\prf(A)$ could be $\{C,E,G\}$ when considering minimal explanations. However, these notions do not say anything about necessity, sufficiency or even relevance. In the next section we therefore look further into restricting the size of explanations, this time based on relevance, sufficiency and necessity.

\section{Necessity and Sufficiency}

Necessity and sufficiency in the context of philosophy and cognitive science are discussed in, for example, ~\cite{Lip90,Lombrozo10,Woodward-6}. %We add here a third: \emph{relevance}, which can be introduced for argumentation frameworks in general. 
Intuitively, an event $\Gamma$ is \emph{sufficient} for $\Delta$ if no other causes are required for $\Delta$ to happen, while $\Gamma$ is \emph{necessary} for $\Delta$, if in order for $\Delta$ to happen, $\Gamma$ has to happen as well. In the context of logical implication (denoted by $\rightarrow$), one could model sufficiency by $\Gamma\rightarrow \Delta$ and necessity by $\Delta\rightarrow \Gamma$~\cite{Lin01}. In the next sections we formulate these logical notions in our argumentation setting. 

\subsection{Necessity and Sufficiency for Acceptance}
\label{sec:NecSuff:Acc}

In the context of argumentation, where explanations are sets of arguments, a set of accepted arguments is sufficient if it guarantees, independent of the status of other arguments, that the considered argument is accepted, while an accepted argument is necessary if it is impossible to accept the considered argument without it. 

\begin{definition}
  \label{def:Acc:NecSuff}
  Let $\calAF = \tuple{\Args,\attack}$ be an AF and let $A\in\Args$ be accepted (w.r.t.\ $\sem$ and $\cup$ or $\cap$). Then:
  \begin{itemize}
      \item $\sfS\subseteq\Args$ is \emph{sufficient for the acceptance} of $A$ if $\sfS$ is relevant for $A$, $\sfS$ is conflict-free and $\sfS$ defends $A$ against all its attackers; 
      \item $B\in\Args$ is \emph{necessary for the acceptance} of $A$ if $B$ is relevant for $A$ and if $B\notin\ext$ for some $\ext\in\exts_\adm(\calAF)$, then $A\notin\ext$. 
  \end{itemize}
\end{definition}

%\begin{example}[Example~\ref{ex:Expl:FullAbstract:Abstract} continued]
%  \label{ex:Acc:NecSuff}
%  Consider again $\calAF_1$ from Example~\ref{ex:abstractAF}. Then $\{A_2\}$ is both necessary and sufficient for the acceptance of $A_4$ and, similarly, $\{B_1\}$ is both necessary and sufficient for the acceptance of $B_3$. 
%\end{example}

\begin{example}[Example~\ref{ex:MonCom} continued]
  \label{ex:Acc:NecSuff:Min}
  In $\calAF_1$ both $\{C\}$ and $\{E\}$ are sufficient for the acceptance of $A$ but neither is necessary, while for $E$, $\{G\}$ is both sufficient and necessary. 
\end{example}

In order to use the above notions as variations of $\depth$, we introduce:\footnote{$\Suff$ and $\Nec$ are defined for an argument and the empty set. We do so because $\depth$ requires an argument and an extension.}
\begin{itemize}
    \item $\Suff(A,\emptyset) = \{\sfS\subseteq\Args\mid \sfS \text{ is }\allowbreak\text{sufficient }\allowbreak\text{for }\allowbreak\text{the }\allowbreak\text{acceptance }\allowbreak\text{of }A\}$ denotes the set of all sufficient sets of arguments for the acceptance of $A$;
    \item $\Nec(A,\emptyset) = \{A\in\Args \mid A \text{ is }\allowbreak\text{necessary }\allowbreak\text{for }\allowbreak\text{the }\allowbreak\text{acceptance }\allowbreak\text{of }A\}$ denotes the set of all arguments that are necessary for the acceptance of $A$. 
\end{itemize} 

\begin{remark}
  \label{rem:SemanticsNecSuff}
  When $\depth\in\{\Suff,\Nec\}$, $\Acc^\star_\sem$ is the same for any semantics $\sem$. This is the case since the definition of sufficiency and necessity is not defined w.r.t.\ $\sem$. 
\end{remark}

\begin{example}[Example~\ref{ex:Acc:NecSuff:Min} continued]
  \label{ex:Acc:NecSuff:Notation}
  For $\calAF_1$ we have that %$\Acc^\cup_\sem(A) \in \{\{C\},\allowbreak  \{E\},\allowbreak \{G\},\allowbreak \{A,C\},\allowbreak \{A,E\}, \allowbreak \{A,G\},\allowbreak  \{C,E\},\allowbreak \{C,G\},\allowbreak \{E,G\},\allowbreak \{A,C,E\},\allowbreak  \{A,C,G\},\allowbreak  \{A,E,G\},\allowbreak  \{C,E,G\},\allowbreak \{A,C,E,G\}\}$ for $\depth = \Suff$ and $\Acc^\cup_\sem(A) = \emptyset$ for $\depth = \Nec$. For $B$ we have that 
  $\Acc^\cup_\sem(B) \in \{\{D,F\},\{B,D,F\}\}$ for $\depth = \Suff$ and $\Acc^\cup_\sem(B) = \{D,F\}$ for $\depth = \Nec$. This means that the credulous acceptance of $B$ can be explained by the existence of the arguments $D$ and $F$, which are both sufficient and necessary for $B$ to be accepted.
\end{example}

Next we show some useful properties of sufficient and necessary (sets of) arguments for acceptance. In particular, we show that the sets in $\Suff(A,\emptyset)$ are admissible and contain all the needed arguments. We further look at conditions under which $\Suff$ and $\Nec$ are empty, as well as the relation between $\Suff$ and $\Nec$.

\begin{restatable}{proposition}{propAccNecSuffSC}
    \label{prop:Acc:NecSuff:SoundCompl}
    Let $\calAF = \tuple{\Args,\attack}$ be an AF and let $A\in\Args$ be accepted w.r.t.\ some $\sem\in\{\adm,\allowbreak \cmp,\allowbreak \grd,\allowbreak \prf,\allowbreak \stb\}$ and $\star\in\{\cap,\cup\}$. Then:
    \begin{enumerate}
        \item For all $\sfS\in\Suff(A,\emptyset)$, $\{\sfS,\sfS\cup\{A\}\}\subseteq\exts_\adm(\calAF)$;
        \item $\Suff(A,\emptyset) = \emptyset$ iff there is no $B\in\Args$ such that $(B,A)\in\attack$;
        \item $\Nec(A,\emptyset) = \emptyset$ iff there is no $B\in\Args$ such that $(B,A)\in\attack$ or $\bigcap\Suff(A,\emptyset) = \emptyset$. 
        \item $\Nec(A,\emptyset)\subseteq\bigcap\Suff(A,\emptyset)$.
    \end{enumerate}
\end{restatable}

\begin{proof}
  Let \(\calAF = \tuple{\Args,\attack}\) be an AF and let \(A\in\Args\) be accepted w.r.t.\ some \(\sem\) and \(\star\in\{\cap,\cup\}\). 
  \begin{enumerate}
      \item Let \(\sfS\in\Suff(A,\emptyset)\). Note that, by definition, \(\sfS\) is conflict-free. If there is some \(B\in\sfS\) such that \((B,A)\in\attack\) then there is some \(C\in\sfS\) that defends against this attack (i.e., \((C,B)\in\attack\)), a contradiction. If, there is some \(B\in\sfS\) such that \((A,B)\in\attack\), then \(A\) indirectly attacks itself. Since there is no \(D\in\sfS\) such that \((D,A)\in\attack\) it follows that \(B\) is not defended against the attack from \(A\). A contradiction with the definition of \(\sfS\) that it defends \(A\) against all attackers. Hence \(\sfS\cup\{A\}\) is conflict-free. 
      
      Now suppose that there is some \(B\in\Args\) and some \(C\in\sfS\) such that \((B,C)\in\attack\). Since \(C\) (in)directly defends \(A\), \(B\) indirectly attacks \(A\). By definition of a sufficient set of arguments \(\sfS\) defends \(A\) against \(B\). It follows that there is some \(D\in\sfS\) such that \((D,B)\in\attack\). Hence \(\sfS\) defends \(A\) and all its own elements against any attacker. Therefore \(\{\sfS,\sfS\cup\{A\}\}\subseteq\exts_\adm(\calAF)\). 
      
      \item Suppose that \(\Suff(A,\emptyset) = \emptyset\), then there is no \(\sfS\subseteq\Args\) such that \(\sfS\) is relevant for \(A\) and defends \(A\) against all its arguments. Since \(A\) is accepted by assumption, it follows that \(A\) is not attacked at all. Now suppose that there is no \(B\in\Args\) such that \((B,A)\in\attack\). Then there is no \(\sfS\subseteq\Args\) that is relevant for \(A\) and hence \(\Suff(A,\emptyset) = \emptyset\). 
      
     \item First suppose that $\Nec(A,\emptyset) = \emptyset$. Then there is no argument relevant for $A$ (from which it follows that there is no $B\in\Args$ such that $(B,A)\in\attack$) or there is no $B\in\Args$ such that $B\in\bigcap\allext_A$. Note that for each $\sfS\in\Suff(A,\emptyset)$ there is some $\ext\in\allext_A$ such that $\sfS\subseteq\ext$. Since $\bigcap\allext_A = \emptyset$ it follows that $\bigcap\Suff(A,\emptyset) = \emptyset$ as well. 
   
     For the other direction suppose first that $A$ is not attacked at all, then there is no argument relevant for $A$ from which it follows that $\Nec(A,\emptyset) = \emptyset$. Now suppose that $\Suff(A,\emptyset) = \emptyset$. By assumption $A$ is accepted and $A$ is attacked, hence $\Suff(A,\emptyset)\neq\emptyset$. It follows that for each $\sfS\in\Suff(A,\emptyset)$ and for each $B\in\sfS$ there is an $\sfS'\in\Suff(A,\emptyset)$ such that $B\notin\sfS'$ and therefore also an $\ext\in\exts_\adm(\calAF)$ with $B\notin\ext$ but $A\in\ext$. Therefore none of the arguments is necessary: $\Nec(A,\emptyset) = \emptyset$. 
     
      \item In view of the above two items, suppose that \(A\) is attacked by some argument. Let \(B\in\Nec(A,\emptyset)\) and suppose that \(B\notin\bigcap\Suff(A,\emptyset)\). Then there is some \(\sfS\in\Suff(A,\emptyset)\) such that \(B\notin\sfS\). Note that \(\sfS\cup\{A\}\in\exts_\adm(\calAF)\). However, \(B\notin\sfS\), a contradiction with \(B\in\Nec(A,\emptyset)\).\qedhere
  \end{enumerate}   
\end{proof}

The next proposition relates the introduced notions of necessity and sufficiency with $\depth = \DefBy$. This shows that, although $\DefBy$ is only one of many options for $\depth$, it is closely related with these selection criteria and therefore a useful notion from which to start the investigation into explanations within the basic framework. 

\begin{restatable}{proposition}{propSuffNecDefBy}
    \label{prop:SuffNecDefBy}
    Let $\calAF = \tuple{\Args,\attack}$ be an AF and let $A\in\Args$ be accepted w.r.t.\ $\sem\in\{\adm,\cmp,\grd,\prf,\stb\}$ and $\star\in\{\cap,\cup\}$. Then:
    \begin{itemize}
        \item for all $\ext\in\allext_A$, $\DefBy(A,\ext)\in\Suff(A,\emptyset)$;
        \item $\bigcap_{\ext\in\allext_A}\DefBy(A,\ext) = \Nec(A,\emptyset)$. 
    \end{itemize}
\end{restatable}

\begin{proof}
  Let $\calAF = \tuple{\Args,\attack}$ be an AF and let $A\in\Args$ be an argument that is accepted w.r.t.\ $\sem$ and $\star\in\{\cap,\cup\}$. Consider both items. 
  \begin{itemize}
      \item Since $A$ is accepted, there is some $\ext\in\exts_\sem(\calAF)$ such that $A\in\ext$. Let $\sfT = \DefBy(A,\ext)$. By definition, $\sfT$ is relevant for $A$ (i.e., all $B\in\sfT$ (in)directly defend $A$ and since \(B\in\ext\), \((B,B)\notin\attack\)). Now suppose that there is some $C\in\Args$ such that $C$ attacks $A$ and $A$ is not defended by $\sfT$. By assumption $A\in\ext$. Hence there is a $D\in\ext$ such that $(D,C)\in\attack$. But then $D$ (in)directly defends $A$ and therefore $D\in\sfT$. Thus $\sfT$ defends $A$ against all its attackers and therefore $\sfT\in\Suff(A,\emptyset)$.%Now suppose that there is some $\ext'\in\exts_\sem(\calAF)$ such that $\sfT\subseteq\ext'$ but $A\notin\ext'$. Then there is some $C\in\Args$ such that $(C,A)\in\attack$ and $\ext'$ does not defend $A$. However, $\ext$ does defend $A$ against $C$, hence there is a $B\in\ext$ such that $(B,C)$ and hence $B\in\sfT$. A contradiction. Therefore $\sfT$ is sufficient for $A$.  
      \item Let $\sfT = \bigcap_{\ext\in\allext_A}\DefBy(A,\ext)$, since $A$ is accepted, $\allext_A \neq\emptyset$. Suppose there is some $B\in\sfT$ which is not necessary for the acceptance of $A$. Then there is an $\ext\in\allext_A$ such that $B\notin\ext$. However, by definition of $\sfT$, $B\in\bigcap \allext_A$. Hence $\sfT$ is necessary for $A$. To see that $\sfT$ contains all the necessary arguments, assume it does not. Then there is some $B\in\Args$ such that $B\notin\sfT$ but $B$ is necessary for the acceptance of $A$. However, since $B\notin\sfT$, there is some $\ext\in\allext_A$ such that $B\notin\ext$, but $A\in\ext$. A contradiction.  \qedhere
  \end{itemize}
\end{proof}

\subsection{Necessity and Sufficiency for Non-Acceptance}
\label{sec:NecSuff:NonAcc}

When looking at the non-acceptance of an argument $A$, the acceptance of any of its direct attackers is a sufficient explanation. However, other arguments (e.g., some of the indirect attackers) might be sufficient as well. An argument is necessary for the non-acceptance of $A$, when it is relevant and $A$ is accepted in the argumentation framework without it. In what follows we will assume that $(A,A)\notin\attack$, since otherwise $A$ itself is the reason for its non-acceptance.

In order to define sufficiency for non-acceptance we need the following definition. 

\begin{definition}
  \label{def:ContestedArgs}
  Let $\calAF = \tuple{\Args,\attack}$ be an AF and let $A,B\in\Args$ such that $A$ indirectly attacks $B$, via $C_1,\ldots,C_n\in\Args$, i.e., $(A,C_1), (C_1,C_2), \ldots, (C_n, B)\in\attack$. It is said that the attack from $A$ on $B$ is \emph{uncontested} if there is no $D\in\Args$ such that $(D,C_{2i})\in\attack$ for $i\in\{1,\ldots,\frac{n}{2}\}$. It is \emph{contested} otherwise, in which case it is said that the attack from $A$ is contested in $C_{2i}$ and that $C_{2i}$ is the contested argument.
\end{definition}

This definition is needed since the acceptance of an indirect attacker might already be sufficient for the non-acceptance of an argument, but not every indirect attacker is sufficient for non-acceptance. See also the next example.

\begin{example}[Example~\ref{ex:Acc:NecSuff:Notation} continued]
  \label{ex:ContestedArgs}
  For $\calAF_1$ from Example~\ref{ex:abstractAF} we have that the indirect attack from $G$ on $B$ is uncontested. This follows since $E$ is not attacked and hence, when $G$ is accepted, so is $E$. Therefore, both $E$ and $G$ can be seen as sufficient for the non-acceptance of $B$. However, the attacks from $D$ and $F$ on $A$ are contested in $B$. For $D$ this follows since it defends $B$, but $B$ is attacked by $E$ and, similarly, $F$ defends $B$, but $B$ is attacked by $C$. Hence, although $D$ and $F$ indirectly attack $A$, by just accepting one, $A$ is not necessarily non-accepted, therefore neither would be sufficient on its own to make $A$ non-accepted.
\end{example}

For the definition of necessary for non-acceptance we define subframeworks, which are needed since an argument might be non-accepted since it is attacked by an accepted or by another non-accepted argument.\footnote{In terms of labeling semantics (see e.g.,~\cite{BCG18}) an argument is non-accepted if it is \texttt{out} (i.e., attacked by an \texttt{in} argument) or \texttt{undecided}.}

\begin{definition}
  \label{def:SubFrameworks}
  Let $\calAF = \tuple{\Args,\attack}$ be an AF and let $A\in\Args$. Then $\calAF_{\downarrow A} = \langle\Args \setminus\{A\}, \attack\cap(\Args\setminus\{A\}\times\Args\setminus\{A\})\rangle$ denotes the AF based on $\calAF$ but without $A$.  
\end{definition}

Since indirect attacks might be sufficient for not accepting an argument, but they also might be contested, the definition of sufficiency for non-acceptance is defined inductively. 

\begin{definition}
  \label{def:NonAcc:NecSuff}
  Let $\calAF = \tuple{\Args,\attack}$ be an AF and let $A\in\Args$ be non-accepted (w.r.t.\ $\sem$ and $\cup$ or $\cap$). Then:
  \begin{itemize}
      \item $\sfS\subseteq\Args$ is \emph{sufficient for the non-acceptance} of $A$ if $\sfS$ is relevant for $A$ and there is a $B\in\sfS$ such that:
    \begin{itemize}
        \item $(B,A)\in\attack$; or
        \item $B$ indirectly attacks $A$ and that attack is uncontested; or
        \item $B$ indirectly attacks $A$ and for every argument $C$ in which the attack from $B$ on $A$ is contested and every $D\in\Args$ such that $(D,C)\in\attack$, there is an $\sfS'\subseteq\sfS$ that is sufficient for the non-acceptance of $D$.
    \end{itemize}
      \item $B\in\Args$ is \emph{necessary for the non-acceptance} of $A$ if $B$ is relevant for $A$ and $A$ is accepted w.r.t.\ $\sem$ and $\cup$ resp.\ $\cap$ in $\calAF_{\downarrow B}$.
  \end{itemize}
\end{definition}

\begin{example}[Example~\ref{ex:ContestedArgs} continued]
  \label{ex:NonAcc:NecSuff}
  For $\calAF_1$ from Example~\ref{ex:abstractAF} we have that $B$ is both necessary and sufficient for the non-acceptance of $A$. Moreover, while $D$ and $F$ are neither sufficient for the non-acceptance of $A$, $\{D,F\}$ is. For the non-acceptance of $B$ we have that $C$, $E$ and $G$ are sufficient, but none of these is necessary. 
\end{example}

We define the following notation, to use the above notions as variations of $\depth$: 

\begin{itemize}
    \item $\SuffNot(A,\emptyset) = \{\sfS\subseteq\Args \mid \sfS \allowbreak\text{ is sufficient }\allowbreak\text{for }\allowbreak\text{the }\allowbreak\text{non-acceptance of }A\}$, denotes the set of sets of arguments that, when accepted, cause $A$ to be non-accepted;
    \item $\NecNot(A,\emptyset) = \{B\in\Args \mid B \allowbreak\text{ is necessary }\allowbreak\text{for }\allowbreak\text{the }\allowbreak\text{non-acceptance of }A\}$, denotes the set of all arguments that are necessary for $A$ not to be accepted.
\end{itemize}

\begin{example}[Example~\ref{ex:NonAcc:NecSuff} continued]
  \label{ex:NonAcc:NecSuff:sets}
  For $\calAF_1$ we have that %$\SuffNot(A,\emptyset) \supseteq \{\{B\}, \allowbreak \{D,F\},\allowbreak  \{B,\allowbreak D,\allowbreak F\}\}$ and $\NecNot(A,\emptyset) = \{\{B\}\}$, 
  $\NotAcc^\cap_\prf(B) \in \{\{C\},\allowbreak\{E\},\allowbreak\{G\},\allowbreak \{C,E\}, \allowbreak \{C,G\}, \{E,G\}, \{C,E,G\}\}$\footnote{The explanation could contain other arguments as well (e.g., the explanation could be $\{A,C,D,E,G\}$). This is the case since $\sfS\in\Suff(B,\emptyset)$ is not assumed to be minimal.} for $\depth = \SuffNot$ while $\NotAcc^\cap_\prf(B) = \emptyset$ for $\depth = \NecNot$. This means that the skeptical non-acceptance of $B$ can be explained by the existence of the arguments $C$, $E$ and $G$, which are all sufficient for the non-acceptance of $B$ but none of them is necessary.
\end{example}

The next propositions are the non-acceptance counterparts of Propositions~\ref{prop:Acc:NecSuff:SoundCompl} and~\ref{prop:SuffNecDefBy}. First some basic properties of sufficiency and necessity for non-acceptance. 

\begin{proposition}
    \label{prop:NotAcc:NecSuff:Empty}
    Let $\calAF = \tuple{\Args,\attack}$ be an AF and let $A\in\Args$ be non-accepted w.r.t.\ $\sem\in\{\adm,\cmp,\grd,\prf,\stb\}$ and $\star\in\{\cap,\cup\}$. Then: 
    \begin{itemize} 
        \item $\SuffNot(A,\emptyset)\neq\emptyset$;
        \item $\NecNot(A,\emptyset) = \emptyset$ implies that there are at least two direct attackers of $A$.
%        \item For any $B\in\NecNot(A,\emptyset)$, $\{B\}\in\SuffNot(A,\emptyset)$.  
    \end{itemize}
\end{proposition}

\begin{proof}
  Let $\calAF = \tuple{\Args,\attack}$ be an AF and let $A\in\Args$ be an argument that is not accepted w.r.t.\ $\sem\in\{\adm,\allowbreak \cmp,\allowbreak \grd,\allowbreak \prf,\allowbreak \stb\}$ and $\star\in\{\cap,\cup\}$. 
  \begin{itemize} 
    \item Suppose that $\Suff(A,\emptyset) = \emptyset$. Then there is no $\sfS\subseteq\Args$ that is relevant for $A$ and in which $B\in\Args$ (in)directly attacks $A$. It follows that there is no $B\in\Args$ such that $(B,A)\in\attack$. A contradiction with the assumption that $A$ is non-accepted and that $(A,A)\notin\attack$. 
    \item It follows that there are $B_1,\ldots,B_n\in\Args$ such that $(B_1,A),\ldots, (B_n,A)\in\attack$. Assume that $\NecNot(A,\emptyset) = \emptyset$ but that $n=1$. Since by assumption in this section $(A,A)\notin\attack$, it follows that $A$ is not attacked in $\calAF_{\downarrow B_1}$ and should therefore be accepted in any complete extension. Hence $n\geq 2$. 
%   \item Let $B\in\NecNot(A,\emptyset)$ and suppose that $\{B\}\notin\SuffNot(A,\emptyset)$. Then $B$ (in)directly attacks $A$. We proceed by induction on $i$, the level of nested contested attacks.
%    \begin{itemize}
%        \item[$i=0$] Then $B$ directly attacks $A$ or $B$ indirectly attacks $A$ and that attack is uncontested. In both cases, $\{B\}\in\SuffNot(A,\emptyset)$, a contradiction. 
%        \item[$i=1$] Then $B$ indirectly attacks $A$, this attack is contested in $C$ and $\{B\}$ is not sufficient for the non-acceptance of all direct attackers of $C$. Since by assumption the attacks from $B$ are not further contested, it follows that there is some $D\in\Args$ such that $(D,C)\in\attack$ and $B$ does not direct or uncontested indirectly attack $D$. But then the attack from $D$ on $C$ is not defended in $\calAF_{\downarrow B}$ in which case the arguments defended by $C$ are no longer defended. Hence $A$ cannot be accepted. A contradiction. 
%        \item[$i={k+1}$] Suppose that for $i$ up to $k$, $\{B\}$ is sufficient for the non-acceptance of the direct attackers of any argument in which the attack from $B$ is contested. Now suppose that for the $k+1$th argument $C$, $B$ does not directly or uncontested indirectly attack one of its direct attackers $D$. Then the attack from $D$ to $C$ is not defended by $B$, from which it follows that the arguments defended by $C$ are not defended in $\calAF_{\downarrow B}$. However, by definition of contested arguments, $C$ \amc{\bf TODO}
%    \end{itemize}
%    It follows that $B$ is sufficient for the non-acceptance of $A$. 
    \qedhere
  \end{itemize}
\end{proof}

Now we show how $\depth = \NotDef$ is related to the here introduced notions of sufficiency and necessity for non-acceptance. For this we first need the following lemma:\footnote{For $\sem = \grd$ this lemma was shown in~\cite{BKT09}.}
 
\begin{restatable}{lemma}{LemRemAttacked}
    \label{lem:RemAttacked}
    Let \(\calAF = \tuple{\Args,\attack}\), \(\ext\in\exts_\sem(\calAF)\) for some \(\sem\in\{\adm,\cmp,\grd,\prf,\stb\}\) and \(A\in\Args\). If there is a \(B\in\ext\) such that \((B,A)\in\attack\), then \(\ext\in\exts_\sem(\calAF_{\downarrow A})\).%\footnote{For \(\sem = \grd\) this lemma was shown in~\cite{BKT09}.} 
\end{restatable}
  
 \begin{proof}
  Let \(\calAF = \tuple{\Args,\attack}\), \(\ext\in\exts_\sem(\calAF)\) for some \(\sem\) and \(A,B\in\Args\) such that \(B\in\ext\) and \((B,A)\in\attack\). Note that \(\ext\) is still admissible in \(\calAF_{\downarrow A}\) since no new attacks are added.
  
  \(\sem \in\{\cmp,\prf\}\). Now suppose there is some \(C\in\Args\) such that \(C\notin\ext\) but \(C\) is defended by \(\ext\) in \(\calAF_{\downarrow A}\). If \(C\) is not attacked at all in \(\calAF_{\downarrow A}\), since \(C\notin\ext\), \((A,C)\in\attack\), but then \(\ext\) defends \(C\) in \(\calAF\), a contradiction. Hence there is some \(D\in\Args\) such that \((D,C)\in\attack\) and \(\ext\) defends against this attack in \(\calAF_{\downarrow A}\), but then \(\ext\) would defend \(C\) in \(\calAF\) as well. Again a contradiction. Hence \(\ext\) is complete in \(\calAF_{\downarrow A}\) and if \(\ext\) was maximally complete in \(\calAF\) it is still maximally complete in \(\calAF_{\downarrow A}\).
  
  \(\sem = \stb\). Any argument, other than \(A\), attacked by \(\ext\) is still attacked by \(\ext\) in \(\calAF_{\downarrow A}\). Since \(\ext\) is still complete, it follows that \(\ext\) is also still stable. 
\end{proof}

\begin{restatable}{proposition}{NotDefNecSuff}
    \label{prop:NotDef:NecSuff}
    Let $\calAF = \tuple{\Args,\attack}$ be an AF and let $A\in\Args$ be an argument that is not accepted w.r.t.\ $\sem\in\{\cmp,\grd,\prf,\stb\}$ and $\star\in\{\cap,\cup\}$. Then:
    \begin{itemize}
        \item for all $\ext\in\exts_\sem(\calAF)$ such that $A\notin\ext$, $\NotDef(A,\ext) \in\SuffNot(A,\emptyset)$;
        \item $\NecNot(A,\emptyset)\subseteq\bigcap_{\ext\in\allext_{\not A}}\NotDef(A,\ext)$.
    \end{itemize}
\end{restatable}

\begin{proof}
  Let $\calAF = \tuple{\Args,\attack}$ be an AF and let $A\in\Args$ be non-accepted w.r.t.\ $\sem\in\{\cmp,\allowbreak \grd,\allowbreak \prf,\allowbreak \stb\}$ and $\star\in\{\cap,\cup\}$. Consider both items:
    \begin{itemize}
        \item By definition of \(\NotDef\), \(\sfT = \NotDef(A,\ext)\) is relevant for \(A\). We show that there is a \(B\in\sfT\) such that \((B,A)\in\attack\). Suppose there is no such \(B\), then \(A\) is not attacked at all or \(\ext\) defends \(A\) against all its direct attackers and therefore against all its attackers, both are a contradiction with the completeness of \(\ext\). Hence there is such a \(B\in\sfT\). From which it follows that \(\NotDef(A,\ext)\in\SuffNot(A,\emptyset)\). 
        \item Let \(B\in\NecNot(A,\emptyset)\) and suppose that \(B\notin\bigcap_{\ext\in\allext_{\not A}}\NotDef(A,\emptyset)\). Then there is some \(\ext\in\allext_{\not A}\) such that \(B\notin\NotDef(A,\ext)\). By assumption, \(B\) is relevant for \(A\) and thus (in)directly attacks \(A\). From which it follows that there is some \(C\in\ext\) such that \((C,B)\in\attack\). By Lemma~\ref{lem:RemAttacked}, \(\ext\in\exts_\sem(\calAF_{\downarrow B})\), a contradiction with the assumption that \(B\in\NecNot(A,\emptyset)\). \qedhere
    \end{itemize}
\end{proof}

%\begin{proof}
%  Let $\calAF = \tuple{\Args,\attack}$ be an AF and let $A\in\Args$ be non-accepted w.r.t.\ $\sem\in\{\cmp,\allowbreak \grd,\allowbreak \prf,\allowbreak \stb\}$ and $\star\in\{\cap,\cup\}$. We show the first item. %Consider both items:
%    %\begin{itemize}
%        %\item 
%        By definition of $\NotDef$, $\sfT = \NotDef(A,\ext)$ is relevant for $A$. We show that there is a $B\in\sfT$ such that $(B,A)\in\attack$. Suppose there is no such $B$, then $A$ is not attacked at all or $\ext$ defends $A$ against all its direct attackers and therefore against all its attackers, both are a contradiction with the completeness of $\ext$. Hence there is such a $B\in\sfT$. From which it follows that $\NotDef(A,\ext)\in\SuffNot(A,\emptyset)$. 
%        %\item Let $B\in\NecNot(A,\emptyset)$ and suppose that $B\notin\bigcap_{\ext\in\allext_{\not A}}\NotDef(A,\emptyset)$. Then there is some $\ext\in\allext_{\not A}$ such that $B\notin\NotDef(A,\ext)$. By assumption, $B$ is relevant for $A$ and thus (in)directly attacks $A$. From which it follows that there is some $C\in\ext$ such that $(C,B)\in\attack$. By Lemma~\ref{lem:RemAttacked}, $\ext\in\exts_\sem(\calAF_{\downarrow B})$, a contradiction with the assumption that $B\in\NecNot(A,\emptyset)$. \qedhere
%    %\end{itemize}
%\end{proof}

\subsection{Necessity, Sufficiency and Minimality}

In order to compare the introduced notions of necessity and sufficiency with the notions of minimality known from~\cite{FaTo15AAAI} and recalled in Section~\ref{sec:Minimality}, we define minimal sufficient sets, where $\preceq\in\{\subseteq,\leq\}$:

\begin{itemize}
    \item $\MinSuff^\preceq(A,\emptyset) = \{\sfS\in\Suff(A,\emptyset)\mid \nexists\sfS'\in\Suff(A,\emptyset)\allowbreak\text{ such that } \sfS'\preceq\sfS\}$, denotes the set of all $\preceq$-minimally sufficient sets for the acceptance of $A$.
    \item $\MinSuffNot^\preceq(A,\emptyset) = \{\sfS\in\SuffNot(A,\emptyset)\mid \nexists\sfS'\in\SuffNot(A,\emptyset)\allowbreak\text{ such that } \sfS'\preceq\sfS\}$, denotes the set of all $\preceq$-minimally sufficient sets for the non-acceptance of $A$.
\end{itemize}
%Moreover, let $\MinDefBy^\preceq(A,\emptyset) = \{\DefBy(A,\ext)\mid \ext\in\allext_A \text{ and } \nexists\ext'\in\allext_A \text{ s.t.\ }\DefBy(A,\ext')\prec\DefBy(A,\ext)\}$ denote the set of all $\preceq$-minimal $\Acc$-explanations with $\depth = \DefBy$ and let $\MinNotDef^\preceq(A,\emptyset) = \{\NotDef(A,\ext)\mid \ext\in\allext_{\not A}\allowbreak \text{ and }\allowbreak  \nexists\ext'\in\allext_A \text{ s.t.\ }\NotDef(A,\ext')\prec\NotDef(A,\ext)\}$ denote the set of all $\preceq$-minimal $\NotAcc$-explanations with $\depth = \NotDef$. 

Although the notions of minimality are aimed at reducing the size of an explanation, by applying instead the notions of sufficiency and necessity as introduced in this paper, the size of the explanation can be further reduced. %, without having to provide partial explanations (e.g., explanations that do not cover all attacking arguments). 
To see this, consider the following example:

\begin{example}
  \label{ex:Acc:NecSuffMin}
  Let $\calAF_2  = \tuple{\Args_2,\attack_2}$, shown in Figure~\ref{fig:Acc:NecSuffMin}. Here we have that $\exts_\prf(\calAF_2) = \{\{A,B\}, \{C,D\}\}$ and, for $\star\in\{\cap,\cup\}$, $\depth^\acc = \DefBy$ and $\depth^\nacc = \NotDef$, that $\Acc^\star_\prf(B) = \{A,B\}$, $\Acc^\star_\prf(D) = \{C,\allowbreak D\}$, $\NotAcc^\star_\prf(B) = \{C,D\}$ and $\NotAcc^\star_\prf(D) = \{A,\allowbreak B\}$. These are the explanations for $B$ and $D$, whether as defined in Section~\ref{sec:Basic} or as in Section~\ref{sec:Minimality}.  %As for each argument there is only one extension in which it is accepted, these are immediately the minimal $\cup$-acceptance and $\cap$-non-acceptance explanations. 
  \begin{figure}[ht]
    \centering
    \begin{tikzpicture}
    \node [draw,circle] at (-2,0) (A) {$A$};
    \node [draw, circle] at (2,0) (B) {$B$};
    \node[draw,circle] at (0,0) (D) {$C$};
    \node[draw,circle] at (4,0) (E) {$D$};
    
    \draw[->] (A) to (D);
    \draw[->] (D) to [bend left=25] (A);
    \draw[->] (A) to [bend left=20] (E);
    \draw[->] (D) to (B);
    \draw[->] (E) to [bend left=25] (B);
    \draw[->] (B) to (E);
    \end{tikzpicture}
    \caption{Graphical representation of $\calAF_2$.}
    \label{fig:Acc:NecSuffMin}
  \end{figure}
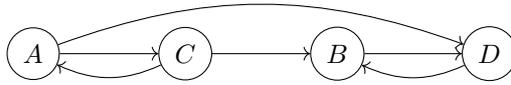
  
  When looking at sufficient sets, we have that $\MinSuff(B,\emptyset) = \{\{A\}\}$, $\MinSuff(D,\emptyset) = \{\{C\}\}$, $\MinSuffNot(B,\emptyset) = \{\{C\}\}$ and $\MinSuffNot(D,\emptyset) = \{\{A\},\allowbreak\{B\}\}$. Therefore, when $\depth^\acc = \MinSuff$ and $\depth^\nacc = \MinSuffNot$, $\Acc^\cup_\prf(B) = \{A\}$ and $\NotAcc^\cap_\prf(D) \in \{\{A\},\allowbreak \{B\}\}$. To see that these explanations are still meaningful, note that $A$ defends $B$ against all of its arguments and as soon as $A$ is accepted under complete semantics, $B$ will be accepted as well. Thus, the minimally sufficient explanations for $B$ and $D$ are $\leq$- and $\subseteq$-smaller than the minimal explanations for $B$ and $D$ from Section~\ref{sec:Minimality}. %Similarly, as soon as $A$ or $B$ is accepted, $D$ cannot be accepted anymore. 
\end{example}

That minimally sufficient explanations can be smaller than minimal explanations is formalized in the next propositions. For this let $\texttt{Acc}^{\star,\preceq}_\sem(A)$ [resp.\ $\texttt{NotAcc}^{\star,\preceq}_\sem(A)$] denote the set of all $\Acc^\star_\sem(A)$ [resp.\ $\NotAcc^\star_\sem(A)$] explanations for $\mima = \min^\preceq$, as defined in Section~\ref{sec:Minimality}. Similarly, let $\texttt{AccS}^{\star,\preceq}_\sem(A)$ [resp.\ $\texttt{NotAccS}^{\star,\preceq}_\sem(A)$] denote the set of all $\Acc^\star_\sem(A)$ [resp.\ $\NotAcc^\star_\sem(A)$] explanations for $\depth = \MinSuff$ [resp.\ $\depth = \MinSuffNot$]. 

\begin{proposition}
    \label{prop:Acc:NecSuffMin}
    Let $\calAF = \tuple{\Args,\attack}$ be an AF, let $A\in\Args$ be accepted w.r.t.\ $\sem\in\{\adm,\allowbreak \cmp,\allowbreak \grd,\allowbreak \prf,\allowbreak \stb\}$ and $\star\in\{\cap,\cup\}$ and let $\preceq\in\{\subseteq,\leq\}$. Then, for $\depth = \DefBy$:
    \begin{itemize}
        \item for every $\sfS\in\texttt{Acc}^{\star,\preceq}_\sem(A)$ there is an $\sfS'\in\texttt{AccS}^{\star,\preceq}_\sem(A)$ such that $\sfS'\preceq\sfS$;
        \item where $\sem = \adm$, for every $\sfS\in\texttt{AccS}^{\star,\preceq}_\sem(A)$ also $\sfS\in\texttt{Acc}^{\star,\preceq}_\sem(A)$;
        \item for all $\sfS\in\texttt{Acc}^{\star,\preceq}_\sem(A)$, $\Nec(A,\emptyset) \subseteq\sfS$. 
    \end{itemize}
\end{proposition}

\begin{proof}
  Let $\calAF = \tuple{\Args,\attack}$ be an AF and let $A\in\Args$ be accepted w.r.t.\ $\sem\in\{\adm,\cmp,\grd,\allowbreak\prf,\stb\}$ and $\star\in\{\cap,\cup\}$. Then, for $\depth = \DefBy$:
  \begin{itemize} 
    \item Let $\sfS\in\texttt{Acc}^{\star,\preceq}_\sem(A)$, then $\sfS = \DefBy(A,\ext)$ for some $\ext\in\allext_A$. By Proposition~\ref{prop:SuffNecDefBy} it follows that $\sfS\in\Suff(A,\emptyset)$. Hence there is some $\sfS'\in\texttt{AccS}^{\star,\preceq}_\sem(A)$ such that $\sfS'\preceq\sfS$, for any of the considered semantics. 
    
    \item Let $\sem = \adm$ and $\sfS\in\texttt{AccS}^{\star,\preceq}_\sem(A)$. By Proposition~\ref{prop:Acc:NecSuff:SoundCompl}, $\sfS\cup\{A\}\in\exts_\adm(\calAF)$ and by definition of a sufficient set of arguments, $\sfS$ defends $A$ against all its attackers. Therefore, $\DefBy(A,\sfS\cup\{A\}) = \sfS$. Suppose that $\sfS\cup\{A\}$ is such that $\DefBy(A,\sfS\cup\{A\})\notin\texttt{Acc}^{\star,\preceq}_\sem(A)$. Then there is some $\ext\in\allext_A^\adm$ such that $\DefBy(A,\ext)\prec\DefBy(A,\sfS\cup\{A\})$. By Proposition~\ref{prop:SuffNecDefBy}, $\DefBy(A,\ext)\in\Suff(A,\emptyset)$. A contradiction since $\sfS\in\texttt{AccS}^{\star,\preceq}_\sem(A)$ and $\DefBy(A,\ext)\prec\sfS$. 
    \item This follows immediately from the second item in Proposition~\ref{prop:SuffNecDefBy}. \qedhere
  \end{itemize}
\end{proof}

%\amc{{\bf TODO:} tegenvoorbeeld tweede item voorgaande propositie.}

\begin{restatable}{proposition}{propNonAccNecSuffMin}
    \label{prop:NotAcc:NecSufMin}
    Let $\calAF = \tuple{\Args,\attack}$ be an AF and let $A\in\Args$ be not accepted w.r.t.\ $\sem\in\{\cmp,\grd,\prf,\stb\}$ and $\star\in\{\cap,\cup\}$. Then, for $\depth = \NotDef$:
    \begin{itemize}
        \item for every $\sfS\in\texttt{NotAcc}^{\star,\preceq}_\sem(A)$ there is an $\sfS'\in\texttt{NotAccS}^{\star,\preceq}_\sem(A)$ such that $\sfS'\subseteq\sfS$.
        \item for all $\sfS\in\texttt{NotAcc}^{\star,\preceq}_\sem(A)$, $\NecNot(A,\emptyset)\subseteq\sfS$.
    \end{itemize}
\end{restatable}

\begin{proof}
  Let $\calAF = \tuple{\Args,\attack}$ be an AF and let $A\in\Args$ be non-accepted w.r.t.\ $\sem\in\{\cmp,\grd,\prf,\stb\}$ and $\star\in\{\cap,\cup\}$. Then, for \(\depth = \NotDef\):
  \begin{itemize}
      \item Let \(\sfS\in\texttt{NotAcc}^{\star,\preceq}_\sem(A)\), then \(\sfS = \NotDef(A,\ext)\) for some \(\ext\in\allext_{\not A}\). By Proposition~\ref{prop:NotDef:NecSuff} it follows that \(\sfS\in\SuffNot(A,\emptyset)\). Hence there is some \(\sfS'\in\texttt{NotAccS}^{\star,\preceq}_\sem(A)\) such that \(\sfS'\preceq\sfS\), for any of the considered semantics. 
      
      \item This follows from the second item in Proposition~\ref{prop:NotDef:NecSuff}.\qedhere
  \end{itemize}
\end{proof}

\section{Related Work}

There exist a few approaches that also aim at providing relevant explanations in terms of sets of arguments. Due to space restrictions we discuss here the main differences, leaving a formal comparison for an extended version of the paper. 

In~\cite{FaTo15AAAI,GCRS13} explanations for accepted arguments are introduced. \cite{FaTo15AAAI} define a new semantics (i.e., related admissibility) to derive explanations for admissible arguments (i.e., arguments that are credulously accepted under admissible semantics) and introduce the notions of minimality and compactness as discussed in this paper. \cite{GCRS13} define explanations for a claim as triples consisting of the dialectical trees that provide a warrant for the claim, dialectical trees that provide a warrant for the contrary of the claim and dialectical trees for the claim or its contrary that do not provide a warrant. 

Explanations for non-accepted arguments are introduced in~\cite{FaTo15TAFA,saribaturWW20}, both focus on credulous non-acceptance under admissible semantics. 
\cite{FaTo15TAFA} introduce argument explanations and attack explanations, such that the considered argument becomes admissible when the arguments/attacks of the explanation are removed from the AF. By their definition these explanations are minimal sets. \cite{saribaturWW20} work with subframeworks, such that the considered argument is non-accepted in the subframework and any of its superframeworks. 

In contrast to each of these approaches, we work with one basic framework that allows for acceptance and non-ac\-cep\-tance explanations and can be applied on top of any Dung-style argumentation framework that is evaluated by some extension-based semantics. The flexibility of this framework (e.g., the use of the function $\depth$) allowed us to study minimality, necessity and sufficiency. As was shown in Propositions~\ref{prop:Acc:NecSuffMin} and~\ref{prop:NotAcc:NecSufMin}, our definitions of necessity and sufficiency allow for even more reduced explanations than the minimality notions from~\cite{FaTo15AAAI}, while still providing meaningful explanations. 

\section{Conclusion}

We have discussed explanations for conclusions derived from AFs in a basic framework that allow for acceptance and non-acceptance explanations, for skeptical and credulous reasoners and for many extension-based semantics. In view of observations from the social sciences~\cite{Mil19} we have studied how the size of such explanations can be reduced in a meaningful way. To this end the notions of minimality from~\cite{FaTo15AAAI} were recalled and we introduced notions of sufficiency and necessity. 

To the best of our knowledge this is the first investigation into necessary and sufficient sets for (non-)acceptance of arguments, especially in the context of integrating 
%
%This paper contains a first investigation into the integration of 
findings from the social sciences (e.g.,~\cite{Mil19}) into (explanations for) argumentation-based conclusions. In future work we plan to investigate how to integrate further findings, such as contrastiveness. Moreover, we will generalize this basic framework as well as the notions of sufficiency and necessity to structured argumentation.

\paragraph{\textbf{Acknowledgements.}} This research was partially funded by the Dutch Ministry of Justice and the Dutch
National Police.

\bibliographystyle{plain}
\bibliography{literature}

\end{document}